\documentclass{article}
\usepackage{textcomp}
\usepackage[utf8]{inputenc}
\usepackage{verbatim}
\usepackage{float}
\usepackage{mathrsfs}
\usepackage{amsmath}
\usepackage{amsthm}
\usepackage{graphicx}
\usepackage[bookmarks=false,
 breaklinks=false,pdfborder={0 0 1},backref=section,colorlinks=false]
 {hyperref}
\hypersetup{
 hidelinks}
\usepackage[switch]{lineno}
\makeatletter

\DeclareTextSymbolDefault{\textquotedbl}{T1}
\providecommand{\tabularnewline}{\\}
\floatstyle{ruled}
\newfloat{algorithm}{tbp}{loa}
\providecommand{\algorithmname}{Algorithm}
\floatname{algorithm}{\protect\algorithmname}

\theoremstyle{plain}
\newtheorem{thm}{\protect\theoremname}
\theoremstyle{definition}
\newtheorem{defn}[thm]{\protect\definitionname}
\theoremstyle{plain}
\newtheorem{lem}[thm]{\protect\lemmaname}
\theoremstyle{plain}
\newtheorem{prop}[thm]{\protect\propositionname}

\usepackage{ijcai26}

\usepackage{times}
\usepackage{soul}
\usepackage{url}
\usepackage[small]{caption}
\usepackage{graphicx}
\usepackage{amsthm}
\usepackage{comment}

\usepackage{textcomp}
\usepackage[utf8]{inputenc}
\usepackage{amsmath}
\usepackage{amsthm}
\usepackage{amssymb}

\makeatletter

\providecommand{\tabularnewline}{\\}

\title{Completeness of Unbounded Best-First Minimax and Descent Minimax}

\author{
Quentin Cohen-Solal\\
\affiliations
LAMSADE, Université Paris-Dauphine, PSL, CNRS, Paris, France\\
\emails
quentin.cohen-solal@dauphine.psl.eu 
}

\makeatother

\providecommand{\definitionname}{Definition}
\providecommand{\lemmaname}{Lemma}
\providecommand{\propositionname}{Proposition}
\providecommand{\theoremname}{Theorem}

\begin{document}
\maketitle
\begin{abstract}
In this article, we focus on search algorithms for two-player perfect
information games, whose objective is to determine the best possible
strategy, and ideally a winning strategy. 

Unfortunately, some search algorithms for games in the literature
are not able to always determine a winning strategy, even with an
infinite search time. This is the case, for example, of the following
algorithms: Unbounded Best-First Minimax and Descent Minimax, which
are core algorithms in state-of-the-art knowledge-free reinforcement
learning. 

They were then improved with the so-called completion technique. However,
whether this technique sufficiently improves these algorithms to allow
them to always determine a winning strategy remained an open question
until now.

To answer this question, we generalize the two algorithms (their versions
using the completion technique), and we show that any algorithm of
this class of algorithms computes the best strategy.

Finally, we experimentally show that the completion technique improves
winning performance.
\end{abstract}
\global\long\def\et{\ \wedge\ }%
\global\long\def\terminal{\mathrm{t}}%
\global\long\def\joueur{\mathrm{j}}%
\global\long\def\joueurUn{\mathrm{1}}%
\global\long\def\joueurDeux{\mathrm{2}}%
\global\long\def\fbin{\mathrm{f_{b}}}%
\global\long\def\fterminal{f_{\mathrm{t}}}%
\global\long\def\fadapt{f_{\theta}}%
\global\long\def\actions{\mathcal{A}}%
\global\long\def\algo{\mathscr{A}}%
\global\long\def\continue{\mathrm{continue}}%
\global\long\def\prioritychildcalculation{\mathrm{calculate\_priority\_child}}%

\global\long\def\etats{\mathcal{S}}%
\global\long\def\Spartiel{\mathcal{S}_{\mathrm{p}}}%
\global\long\def\ubfm{\mathrm{UBFM}}%
\global\long\def\ubfms{\ubfm_{\mathrm{\mathrm{s}}}}%
\global\long\def\umaxn{\ubfm^{n}}%
\global\long\def\maxn{\mathrm{Max}^{n}}%
\global\long\def\umaxns{\ubfms^{n}}%
\global\long\def\descente{\mathrm{Descent}}%
\global\long\def\descenten{\descente^{n}}%
\global\long\def\fbinn{\fbin}%
\global\long\def\fterminaln{\fterminal}%
\global\long\def\fadaptn{\fadapt}%
\global\long\def\ou{\,\vee\,}%
\global\long\def\ubfm{\mathrm{UBFM}}%
\global\long\def\argmax{\operatorname*{\mathrm{arg\,max}}}%
\global\long\def\argmin{\operatorname*{\mathrm{arg\,min}}}%
\global\long\def\liste#1#2{\left\{  #1\,|\,#2\right\}  }%
\global\long\def\minimum{\operatorname*{\mathrm{min}}}%

\section{Introduction}

Unbounded Best-First Minimax~\cite{korf1996best,cohen2020minimax,cohen2019apprendre}
is an old game tree search algorithm which is rarely used. Unlike
classic Minimax, this algorithm performs a search at an unbounded
depth, making it possible to explore the game tree non-uniformly.
Thus, it can anticipate many future turns in the most promising parts
of the game tree. Most recently, it has been successfully applied
in the context of \emph{reinforcement learning without knowledge}~\cite{cohen2023minimax,2020learning}.
More precisely, for many games, at least combined with a search algorithm
dedicated to reinforcement learning, named \emph{Descent Minimax},
Unbounded Minimax gives much better results than the old state of
the art of reinforcement learning without knowledge (i.e. the AlphaZero
algorithm~\cite{silver2018general}, based on Monte Carlo Tree Search
(MCTS)~\cite{Coulom06,browne2012survey}, another search algorithm,
which is stochastic). A thorough experimental study showed that Unbounded
Minimax was the best-performing search algorithm~\cite{cohen2025study}. 

The Descent (Minimax) algorithm is a modification of Unbounded Best-first
Minimax that builds the search tree focusing more on depth-first exploration
in order to more effectively propagate endgame information during
learning. 

In their basic versions, Unbounded Minimax and Descent are not complete
in certain contexts such as reinforcement learning, that is to say
they do not allow one to compute an equilibrium point of the complete
game tree (i.e. a winning strategy for each player), even in infinite
time \cite{2020learning}. For example, the basic algorithms, not
using the fact that certain states are \textit{resolved}, will always
choose to play during the exploration a resolved state rather than
an unresolved state whose evaluation is worse than the resolved state
(whereas the evaluation of an unresolved state is only an estimate,
so the unresolved state may be better). This blocks the exploration
and prevents, when this scenario occurs, calculating the minimax value
and therefore determining the best action to play.

Being able to determine the best strategy is crucial. In a state where
one has the advantage, determining the best strategy guarantees us
to keep this advantage until the end of the game and therefore to
win the match. In a state where the opponent has the advantage, determining
the best strategy can make it possible to exploit an error in the
opponent's strategy to take advantage and thus win the game (if the
latter has not determined the best strategy).

Although it is usually impractical to determine the best strategy
in the early game, it is possible to do so in the late game. The objective
of the match is therefore to play as well as possible from the start,
so as to reach a sufficiently favorable state in the late game where
a winning strategy can be determined.

In \cite{2020learning}, an improvement of Descent and Unbounded Minimax,
which is called completion, has been proposed. However, it was not
clear whether this modification made both algorithms complete, and
more generally if using Completion is actually useful in practice

Thus, we show in this article that  completion makes the two algorithms
complete in finite time. More precisely, on the one hand, we identify
a whole class of search algorithms, grouping Unbounded Best-First
Minimax and Descent. On the other hand, we show that each algorithm
belonging to this class always computes, for any two-player game with
perfect information, one of its best strategies (in finite time).
Finally, we show that the use of completion actually and significantly
improves the practical winning performance of Unbounded Minimax, showing
that it is important to use Completion.

We start by making, in the next section, a quick survey about search
algorithms for two-player perfect-information games. In Section \ref{games_and_completeness},
we formally define what we call two-player games with perfect information
and the fact that a search algorithm is complete (that is, it calculates
the best strategy). In particular, we explain the notations used.
In Section \ref{state_of_the_art_algorithms}, we present the two
algorithms: Unbounded Best-First Minimax and Descent using the completion
technique. Then, in Section \ref{gum}, we generalize these two algorithms
within a particular class of search algorithms, that we call Unbounded
Minimax-based algorithms. Moreover, in Section \ref{proofs}, we show
that Unbounded Minimax-based algorithms are all complete: they effectively
compute the best strategy (proofs are in Appendix). Finally, in Section~\ref{sec:Experimental-Comparison},
we experimentally show that the completion technique improves winning
performance for Unbounded Best-First Minimax.

\section{Related Work}

Artificial Intelligence in Games~\cite{millington2018artificial,bouzy2020artificial,mandziuk2010knowledge,allis1994searching,yannakakis2018artificial,russell2016artificial}
has a long research history.

The basic search algorithm in two-player perfect information games
is minimax~\cite{morgenstern1953theory}. It builds a search tree
whose nodes are states of the game linked to each other when they
are reachable by an action. There are $\max$ nodes, i.e. states where
the first player plays and seeks to maximize his gain and $\min$
nodes, i.e. states where the second player plays and seeks to minimize
the gain of the first player (we place ourselves in the context of
zero-sum games). The search is carried out up to a depth $p$ fixed
in advance. The minimax algorithm then recursively computes the \emph{approximate
minimax value associated with the search tree of depth $p$ }for each
state in that tree. The value of a non-terminal leaf state is given
by an \emph{evaluation function} which provides an indication on how
promising this state is or not, in order to compare the states with
each other. The value of a terminal leaf state is its gain (usually
$1$ if the first player wins, $0$ for a draw and $-1$ if the first
player loses; the gain can also be the difference of player scores).
The value of a $\max$ state is the maximum of the value of its child
states and the value of a $\min$ state is the minimum of the value
of its child states. Note that the \emph{exact minimax value} corresponds
to the ``infinite'' depth, i.e. when all the descendants of all
the states of the tree are in the search tree. In that case, each
leaf is a terminal state. Therefore, an evaluation function gives
a better level of play when it constitutes a good approximation of
the exact minimax value of states. Once the search tree has been constructed
from a game state, the best action in this state is the one that leads
to the child state with the best approximate minimax value (the maximum
value child for a max state and the minimum value child for a min
state).

Many improvements and variants of minimax have been proposed, in particular
to reduce the number of states to be analyzed (notably by reducing
the branching factor), in order to speed up the calculation of the
best action (which is exponential with respect to the depth of the
search). The standard improvement is the Alpha-Beta algorithm~\cite{knuth1975analysis,baudet1978branching,fink1982enhancement,pearl1982solution}
which performs exact pruning during the minimax search. It reduces
the number of states to be evaluated without changing the strategy
determined by the search. More precisely, it infers that some parts
of the search tree do not influence the minimax value of the root
of the search tree and are therefore unnecessary for the calculation
of the best strategy.

Among the variants of minimax, there is the family of fixed-depth
best-first minimax~\cite{plaat1996best,plaat2017minimax}, and in
particular the MTD(f) algorithm, which iteratively performs alpha-beta
searches by assuming that the minimax value is in a specific interval
in order to perform strong pruning and thus obtain an inexpensive
search. This interval is modified at each iteration, in order to finally
find the minimax value at depth $p$.

Another variant is the Principal Variation Search algorithm~\cite{pearl1980asymptotic,pearl1980scout,qi2020optimization}
which checks that the states are well ordered by their value, in particular
the state of best value. The verification is less expensive than applying
the alpha-beta algorithm. If the verification fails, a standard alpha-beta
search is used to determine the minimax value. Different heuristics
were also used to modify the search: for example the null move heuristic~\cite{donninger1993null}
and the history heuristic~\cite{schaeffer1983history}, as well as
move ordering technique \cite{fink1982enhancement}, which influences
the pruning performance of alpha-beta. Moreover, there is also the
quiescence search~\cite{beal1990generalised} which searches until
the leaves of the tree are \textquotedbl quiet\textquotedbl .

There are also unbounded algorithms, i.e. with an unfixed depth. They
make it possible to mitigate the problem of the combinatorial explosion
by building a non-uniform game tree, which permits to explore very
in depth locally and to perform powerful but possibly temporary pruning.
Proof Number Search first explores the \emph{a priori} faster-to-resolve
states. B{*}~\cite{berliner1981b} uses two ratings: an optimistic
evaluation and a pessimistic evaluation and seeks to show that the
pessimistic value of the \emph{a priori} best state is better than
the optimistic values of the other states. Conspiracy search~\cite{mcallester1988conspiracy,schaeffer1990conspiracy}
consists in exploring the states that most affect the minimax value
first. Finally, unbounded best-first minimax first explores the states
of \emph{a priori} best minimax value. This algorithm requires a good
evaluation function to give good results: the better the evaluation
function is, the more powerful and useful the prunings are and the
better the level of play is. But conversely, worse evaluation functions
lead to an even worse level of play than with other algorithms. For
this reason, this algorithm has a limited practical use, due to the
lack of a learning procedure providing a sufficiently good evaluation
function (until~\cite{cohen2019apprendre}). The Descent search is
a variant of Unbounded Best-First Minimax which explores the game
tree by focusing much more on depth-first exploration. The objective
of this algorithm is not only to calculate the best possible approximate
minimax value of the root of the search tree in limited time but also
to provide better quality data for learning.

Opposing minimax-type approaches, there is Monte Carlo Tree Search~\cite{browne2012survey,Coulom06,chaslot2008monte,gelly2011monte,kocsis2006bandit,kocsis2006improved}
which replaces the minimax value of a state by the victory statistics
of the game simulations carried out during the search passing through
this state. It uses a regret term to bias the search in order to manage
the exploration-exploitation dilemma. Note that it has been proven
that the standard MCTS is complete in infinite time~\cite{kocsis2006bandit,kocsis2006improved}.
It requires for this to fully explore the game tree. This is not the
case for Unbounded Minimax and Descent thanks to their exact pruning.
A variant of MCTS has been proposed which is complete in finite time
\cite{winands2008monte}. However, this variant does not manage resolved
draws (contrary to Unbounded Minimax and Descent with completion).
In addition, there are over-estimates and under-estimates problems
because pruning from resolution biases values of states in MCTS (see
Section 4.2 of \cite{winands2008monte}). To mitigate this problem,
they use an arbitrary parameter as a threshold value and it requires
that some resolved states are explored again and several times. This
problem and the need to explore a resolved state again do not appear
with Unbounded Minimax and Descent with completion.

Monte Carlo Tree Search is the search of the AlphaZero algorithm,
which has managed to reach the level of the best human players in
Go and Chess~\cite{silver2018general,silver2017mastering}. At least
for accessible learning durations, the Athénan algorithm~\cite{cohen2019apprendre}
based on Unbounded Best-First Minimax and Descent Minimax gives better
performance than the AlphaZero algorithm (at least thirty times faster)~\cite{cohen2023minimax}.
This was confirmed at international tournaments: at the Computer Olympiad,
Athénan shattered all records, overwhelmingly defeating all competitors,
including those using AlphaZero-based approaches. The record for gold
medals in a single year is threefold~\cite{cohen2023athenan}. Since
its introduction in 2020, it has won 57 gold medals and currently
holds the title in 20 games~\cite{2020learning,cohen2021descent,cohen2023athenan}.
At the most recent competition, it won the gold medal in 9 out of
11 games.

Finally, there are also approaches hybridizing minimax and Monte Carlo
Tree Search~\cite{baier2013monte,baier2015mcts,baier2018mcts,lanctot2014monte,lin2017monte}.

Note that the labelings of states carried out by search algorithms
are generally stored in hash tables, called \emph{transposition tables}.

\section{Perfect Two-Player Games and Completeness }\label{games_and_completeness}

In this section, we are interested in formalizing zero-sum two-player
games with perfect information (games without hidden information or
chance where players take turns playing). Moreover, we detail the
values used by the completion technique and present the other notations
employed in this document.

\subsection{Perfect Two-Player Game}

We start by defining two-player games with perfect information:
\begin{defn}
A perfect two-player game is a tuple $\left(S,\actions,\joueur,\fbin\right)$
where 
\begin{itemize}
\item $\left(\etats,\actions\right)$ is a finite lower semilattice, 
\item $\joueur$ is a function from $\etats$ to $\left\{ 1,2\right\} $, 
\item $\fbin$ is a function from $\liste{s\in\etats}{\terminal\left(s\right)}$
to $\left\{ -1,0,1\right\} $, where 
\begin{itemize}
\item $\terminal$ is a predicate such that $\terminal\left(s\right)$ is
true if and only if $\left|\actions\left(s\right)\right|=0$ ; 
\item $\actions\left(s\right)$ is the set defined by $\liste{s'\in\etats}{\left(s,s'\right)\in\actions}$,
for all $s\in\etats$. 
\end{itemize}
\end{itemize}
\end{defn}

The set $\etats$ is the set of game states. $\actions$ encodes the
actions of the game: $\actions\left(s\right)$ is the set of states
reachable by an action from $s$. The function $\joueur$ indicates,
for each state, the number of the player whose turn it is, i.e. the
player who must play. The predicate $\terminal\left(s\right)$ indicates
if $s$ is a terminal state (i.e. an end-of-game state). Let $s\in\etats$
such that $\terminal\left(s\right)$ is true, the value $\fbin\left(s\right)$
is the payout for the first player in the terminal state $s$ (the
gain for the second player is $-\fbin\left(s\right)$ ; we are in
the context of zero-sum games). We have $\fbin\left(s\right)=1$ if
the first player wins, $\fbin\left(s\right)=0$ in the event of a
draw, and $\fbin\left(s\right)=-1$ if the first player loses.

The terminal evaluation $\fbin$ is often not very informative about
the quality of a game. A more expressive terminal evaluation function,
with values in $\mathbb{R}$, can be used to favor some games over
others, which can improve the level of play \cite{2020learning}.
We denote such functions by $\fterminal$. For example, in score games,
$\fterminal(s)$ may be the score of the endgame $s$ (see \cite{2020learning}
for other terminal evaluation functions). In order to guide the search,
it is necessary to be able to evaluate non-terminal states. To do
this, an evaluation function from $\etats$ to $\mathbb{R}$, denoted
by $\fadapt(s)$, is used. A good evaluation function $\fadapt(s)$
provides for example an approximation of the minimax value of $s$.
Such a function $\fadapt(s)$ can be determined by reinforcement learning,
using $\fterminal(s)$ as ``reinforcement heuristic'' \cite{2020learning}.

\subsection{State Values of Completion}\label{subsec:States-Values-of}

We now detail the foundations needed to understand the completion
technique. With the completion technique, each state $s$ of the partial
game tree constructed by the used search algorithm is associated with
three values. The first value, the completion value $c\left(s\right)$,
indicates the exact minimax value of $s$ compared to the classic
gain of the game (if it is not known, $c\left(s\right)=0$). The (exact)
minimax value of a state $s$ with respect to a certain terminal evaluation
function (such as the classic game gain) is the terminal evaluation
of the end-of-game state obtained by starting from $s$ where the
players play optimally (i.e. the first player maximizes the end-of-game
value and the second player minimizes this value ; each by having
complete knowledge of the full game tree). 

The second value, the heuristic evaluation $v\left(s\right)$, is
an estimate of the true minimax value of $s$ with respect to a certain
reinforcement heuristic (a terminal evaluation function, at least,
more expressive than the classic gain of the game). More precisely,
the heuristic evaluation $v\left(s\right)$ is the minimax value of
$s$ in the partial game tree where the terminal leaves of the tree
are evaluated by the reinforcement heuristic and the other leaves
by an adaptive evaluation function (learned, for example, by reinforcement
learning, to estimate the minimax value corresponding to the reinforcement
heuristic). 

Finally, the third value, the resolution value $r\left(s\right)$,
indicates whether the state $s$ is (weakly) resolved (case $r\left(s\right)=1$)
or not (case $r\left(s\right)=0$). A terminal state $s$ is resolved
(thus at any time $r\left(s\right)=1$). A non-terminal leaf $s$
is not resolved at the start of the algorithm (initially, $r\left(s\right)=0$).
Then, an internal (i.e., non-leaf) state is resolved if its best child
is a winning resolved state for the current player or if all of its
children are resolved ($c\left(s\right)$ is then exact). 

\subsection{Other notations}

We describe the other notations used in this article. First, $n(s,s')$
is the number of times a child state $s'$ is selected from a parent
state $s$ (has no impact on the completeness) ; $\Spartiel$ : set
of explored states by the used search algorithm, i.e. the partial
game tree. It is in practice the keys of the transposition table $T$,
associating a state $s$ with the values $v\left(s\right)$, $c\left(s\right)$,
and $r\left(s\right)$. $\top$: True, $\bot$: False.

\section{State-of-the-Art Algorithms }\label{state_of_the_art_algorithms}

\subsection{Unbounded Best-First Minimax with Completion }

Unbounded (Best-First) Minimax (noted more succinctly $\ubfm$) is
an algorithm which builds a (partial) tree of the game to decide which
is the best action to play given the current knowledge about the game
(i.e. given the partial tree of the game). The algorithm iteratively
builds the partial tree of the game by extending each time the best
sequence of unresolved states (in a state $s$ the first player (resp.
second player) chooses the state $s'$ which maximizes (resp. minimizes)
the lexicographically ordered pair $\left(c\left(s\right),v\left(s\right)\right)$).
In other words, it adds the child states of the principal variation
of the partial game tree deprived of resolved states. An iteration
of Unbounded Minimax is described in Algorithm~\ref{alg:ubfms_search}.
This iteration is performed (on a game state) as long as there is
some search time left ($\tau$: search time) or until this state is
resolved (and therefore its minimax value is determined, as we will
prove later). Note that to decide the action to play after the search,
the best action is chosen (see Algorithm~\ref{alg:descente-completion-suite}).

\begin{algorithm}[!bh]
\begin{centering}
\includegraphics[scale=0.5]{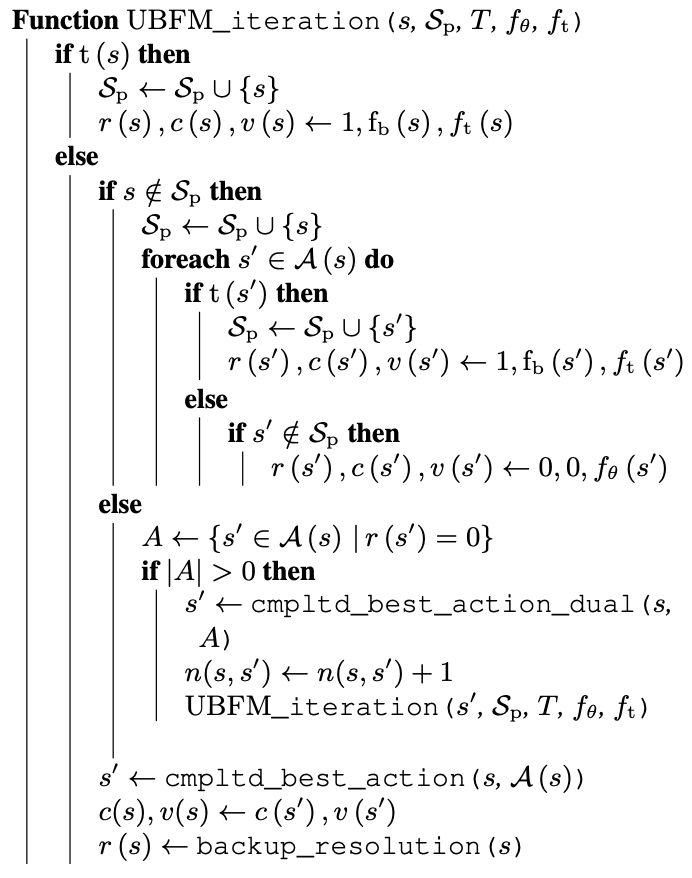}
\par\end{centering}
\caption{\emph{$\protect\ubfm$} tree search algorithm with completion (see
Algorithm~\ref{alg:descente-completion-suite} for the definitions
of completed\_best\_action($s$) and backup\_resolution($s$)). }\label{alg:ubfms_search}
\end{algorithm}

\subsection{Descent with Completion}

Descent Minimax is a variant of Unbounded Minimax. The key difference
is that it extends the best sequence of unresolved game states until
it reaches either a terminal or a resolved state. As a result, it
performs deterministic endgame simulations iteratively.

In contrast, Unbounded Minimax stops its iteration as soon as it extends
a new state. It therefore extends the current best sequence of actions
by only a single action at each iteration.

Descent Minimax explores the game tree in a different order, with
a much stronger depth-first bias. This exploration priority is particularly
interesting from a learning perspective~\cite{2020learning}. An
iteration of Descent Minimax is described in Algorithm~\ref{alg:descente-completion}.

\begin{algorithm}[!bh]
\begin{centering}
\includegraphics[scale=0.5]{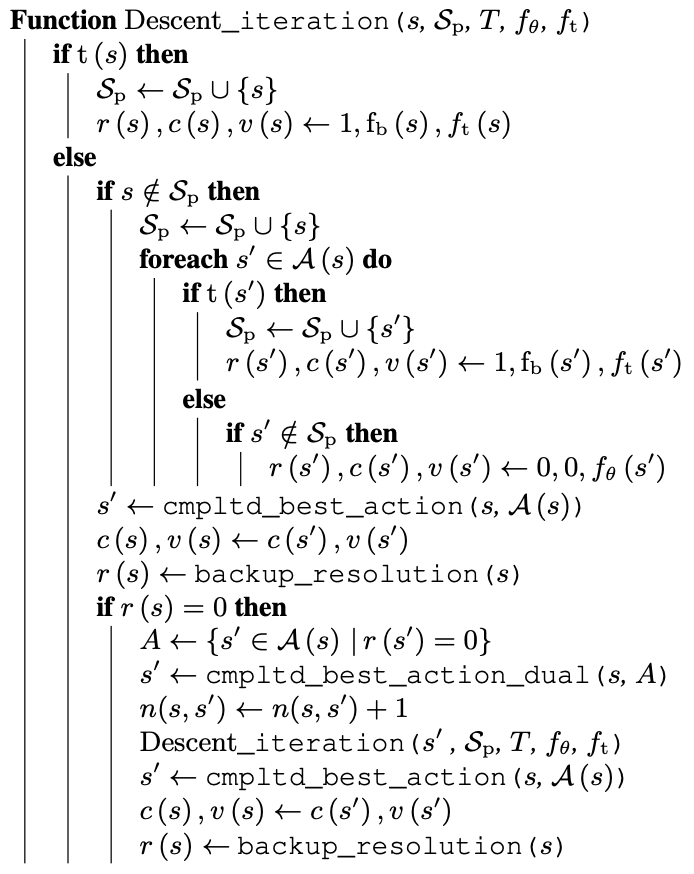}
\par\end{centering}
\caption{\emph{Descent Minimax} tree search algorithm with completion (see
Algorithm~\ref{alg:descente-completion-suite} for the definitions
of completed\_best\_action($s$) and backup\_resolution($s$)). }\label{alg:descente-completion}
\end{algorithm}

\begin{algorithm}[!bh]
\begin{centering}
\includegraphics[scale=0.5]{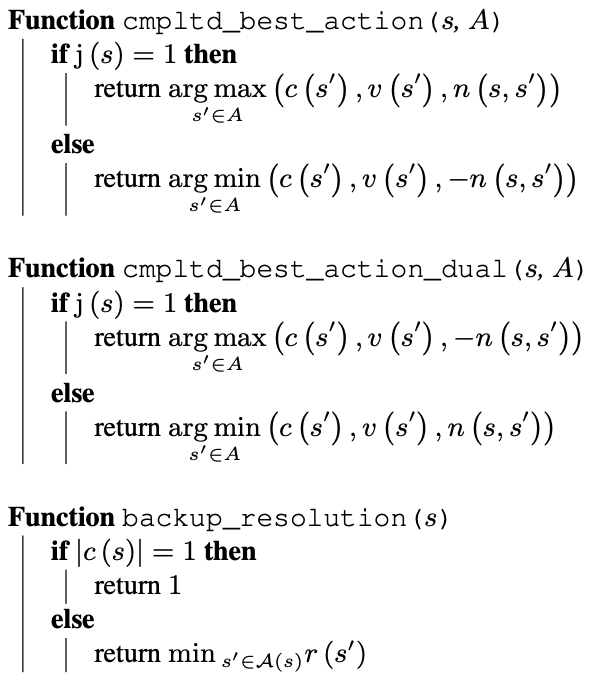}
\par\end{centering}
\caption{Definition of the algorithms completed\_best\_action($s$, $A$),
which computes the \emph{a priori }best child state by using completion
from a set of child states $A$, and backup\_resolution($s$), which
updates the resolution of $s$ from its child states. The methods
use lexicographic ordering.}\label{alg:descente-completion-suite}
\end{algorithm}

\section{Unbounded Minimax-based algorithms}\label{gum}

We now generalize the Descent (Minimax) and Unbounded (Best-First)
Minimax algorithms using the completion technique into a set of search
algorithms, preserving its main properties (they are minimax-like
algorithms operating without a depth bound). However, each algorithm
in this set explores the game tree differently. More precisely, this
class of algorithms groups together all the possibilities of behavior
with regard to the decision to continue the exploration in depth of
the line of play being currently analyzed or to return to the root
to possibly choose a new line of play. Unbounded Best-First Minimax
always returns to the root after having extended by only one state
the line of play and Descent returns to the root only when the extended
line of play reaches the end of the game. This class of algorithms
also includes all the possibilities of behavior for the order of exploration
of the different children of a state. Unbounded Best-First Minimax
and Descent always choose the best child.

We call this class of algorithms  \emph{Unbounded Minimax-based algorithms}.
Note that any Unbounded Minimax-based algorithm natively uses the
completion technique. Thus, this class includes the completion versions
of Descent and Unbounded Best-First Minimax. Recall that without the
completion technique, these two algorithms are not complete and we
show in the sequel that they are complete with the completion technique.

The formalization of an iteration of a Unbounded Minimax-based algorithm
is described in Algorithm \ref{alg:uum}. This iteration is performed
(on a game state) as long as there is some search time left ($\tau$:
search time) or until this state is resolved (and therefore its minimax
value is determined, as we will prove later). This “abstract” iteration
algorithm depends on two methods which must be specified in order
to instantiate the algorithm. The method $\continue$ determines if
the algorithm continues to expand the current line of play (i.e. to
explore more deeply the estimation of the rest of the match) or if
the algorithm considers a potentially new line of play starting from
the current state of the match (i.e. potentially considers another
estimate of the rest of the match). This method defines the order
of “deep exploration”. The method $\prioritychildcalculation(s)$
returns the child state of $s$ to be explored first. This second
method defines the order of “breadth exploration”.

A Unbounded Minimax-based algorithm builds a (partial) tree of the
game to decide which is the best action to play given the current
knowledge about the game (i.e. given the partial tree of the game).
Each state $s$ of the partial game tree is associated with the three
completion values (see Section~\ref{subsec:States-Values-of}), like
with Unbounded Best-First Minimax and Descent.

A Unbounded Minimax-based algorithm iteratively builds the partial
tree of the game by extending each time a sequence of unresolved states
(which is not necessarily the best sequence of unresolved states contrary
to what Descent and Unbounded Best-First Minimax do). Recall that
Unbounded Best-First Minimax extends the sequence of states by a single
state, which is the best child state. In contrast, Descent recursively
extends the sequence of states by the best child state until a terminal
state is reached. More generally, a Unbounded Minimax-based algorithm
recursively extends the sequence of states by one of the child states,
which is not necessarily the best child state, up to a certain state,
which is not necessarily a terminal state.

\begin{algorithm}[!th]
\begin{centering}
\includegraphics[scale=0.5]{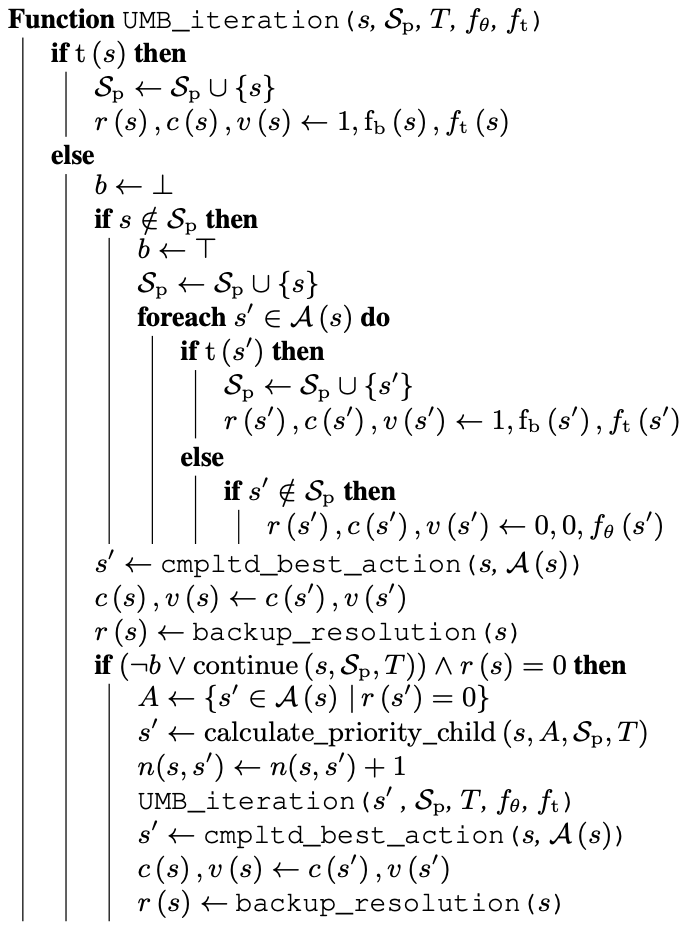}
\par\end{centering}
\caption{Iteration of a Unbounded Minimax-based algorithm (see Algorithm~\ref{alg:descente-completion-suite}
for the definitions of completed\_best\_action($s$) and backup\_resolution($s$)).
It is based on two abstract methods: $\protect\prioritychildcalculation(s,A,\protect\Spartiel,T)$
which must return one of the children of $s$ in $A$ depending on
the current search status (i.e. $\protect\Spartiel$ and $T$) and
$\protect\continue(s,\protect\Spartiel,T)$ which must return a boolean
that determines whether to continue to explore the current line of
play deeper or backtrack to the root to potentially explore a new
line of play, depending on the current search status (i.e. $\protect\Spartiel$
and $T$). Note: $b$ is True if $s$ has just been extended.}\label{alg:uum}
\end{algorithm}

\section{Completeness of Unbounded Minimax-Based Algorithms}\label{proofs}

We start by precisely formalizing what we mean by completeness in
the context of a Unbounded Minimax-based algorithm. For this, we define
the (exact) minimax value of a state. 
\begin{defn}
Let $\etats$ be a complete game tree.

The \emph{minimax value of a state $s\in\etats$ (compared to the
terminal evaluation $\fbin$)} is the value $M\left(s\right)$ recursively
defined by 
\[
M\left(s\right)=\begin{cases}
\max_{s'\in\actions\left(s\right)}M\left(s'\right) & \text{if }\lnot\terminal\left(s\right)\et\joueur\left(s\right)=\joueurUn\\
\min_{s'\in\actions\left(s\right)}M\left(s'\right) & \text{if }\lnot\terminal\left(s\right)\et\joueur\left(s\right)=\joueurDeux\\
\fbin\left(s\right) & \text{if }\terminal\left(s\right)
\end{cases}
\]
\end{defn}

We thus define completeness as the ability to compute the minimax
value (in finite time).
\begin{defn}
A Unbounded Minimax-based algorithm is \emph{complete} if for any
perfect two-player game $\left(\etats,\actions,\joueur,\fbin\right)$
and for any state $s\in\etats$, there exists $\tau\in\mathbb{N}^{+}$,
such that applying the algorithm on $s$, with $\tau$ as search time,
gives $r\left(s\right)=1$ and $c\left(s\right)=M\left(s\right)$. 
\end{defn}

We can now show that any Unbounded Minimax-based algorithm is complete.

Note first that if a state is marked as resolved (i.e. $r(s)=1$),
its values do not change:
\begin{lem}
\label{lem:Si-avant} Let $\hat{s}$ be a state of a perfect two-player
game $G$.

If $r\left(\hat{s}\right)=1$ before an iteration of a Unbounded Minimax-based
algorithm on a certain state $\hat{s}'$ of $G$ then after the iteration,
$r\left(\hat{s}\right)$ and $c\left(\hat{s}\right)$ have not changed. 
\end{lem}

Note now that a state $s$ is marked as resolved (i.e. $r(s)=1$)
if and only if all of its children are marked as resolved (any child
state $s'$ satisfies $r(s')=1$) or $s$ is marked as winning for
the first player ($c(s)=1$) or it is marked as losing for the first
player ($c(s)=-1$):
\begin{lem}
\label{lem:completion-resolution}Let $\left(\Spartiel,\actions\right)$
be a game tree built by a Unbounded Minimax-based algorithm from a
certain state. Let $s\in\Spartiel$. We have the following properties: 
\begin{itemize}
\item If $r\left(s\right)=1$, then either $\left|c\left(s\right)\right|=1$
or for all $s'\in\actions\left(s\right)$, $r\left(s'\right)=1$ ; 
\item If $\left|c\left(s\right)\right|=1$, then $r\left(s\right)=1$ ; 
\item If for all $s'\in\actions\left(s\right)$, $r\left(s'\right)=1$,
then $r\left(s\right)=1$. 
\end{itemize}
\end{lem}

We now show that if a state $s$ is marked as resolved ($r(s)=1$)
then it is resolved, i.e. its completion value is equal to its minimax
value ($c\left(s\right)=M\left(s\right)$):
\begin{prop}
\label{prop:exact}Let $\left(\Spartiel,\actions\right)$ be a game
tree built by a Unbounded Minimax-based algorithm from a certain state
and let $s\in\Spartiel$.

If $r\left(s\right)=1$, then $c\left(s\right)=M\left(s\right)$. 
\end{prop}

In the lemma below, we show under which sufficient conditions an iteration
of a Unbounded Minimax-based algorithm terminates.
\begin{lem}
\label{lem:arret}Let $\algo$ be a Unbounded Minimax-based algorithm
and $s'$ be a non-terminal state of a perfect two-player game such
that $r(s')=0$.

An iteration of the algorithm $\algo$ on the state $s'$ ends in
a state $s$ such that either $s\notin\Spartiel$ before the iteration
and $s\in\Spartiel$ after the iteration or $r(s)=0$ before the iteration
and $r(s)=1$ after the iteration.
\end{lem}

We now show that all Unbounded Minimax-based algorithm mark any game
state as resolved after a finite number of iterations:
\begin{prop}
\label{prop:arret} Let $\algo$ be a Unbounded Minimax-based algorithm,
$\etats$ be the set of states of a perfect two-player game, and $s\in\etats$.

There exists $N\in\mathbb{N}$ such that after applying $N$ times
the algorithm $\algo$ on the state $s$, we have $r\left(s\right)=1$. 
\end{prop}

From the previous Lemmas and Propositions, we conclude to the completeness
of Unbounded Minimax-based algorithms:
\begin{thm}
Any Unbounded Minimax-based algorithm is complete. 
\end{thm}

\section{Experimental Comparison}\label{sec:Experimental-Comparison}

{\small{}
\begin{table*}[!t]
\begin{centering}
{\small\centering{} \centering{}\caption{Detailed performance of Unbounded Best-First Minimax with completion
against Unbounded Best-First Minimax without completion.}\label{tab:Detailed-performance-of}
}{\small\par}
\par\end{centering}
\centering{}{\tiny{}%
\begin{tabular}{ccccccccccc}
\hline 
{\tiny{}Amazons} & {\tiny{}Arimaa} & {\tiny{}Ataxx} & {\tiny{}Breakthrough} & {\tiny{}Brazilian} & {\tiny{}Canadian} & {\tiny{}Clobber} & {\tiny{}Connect6} & {\tiny{}International} & {\tiny{}Chess} & {\tiny{}Go 9}\tabularnewline
\hline 
{\tiny 1\% \textpm 2\%} & {\tiny{} 7\% \textpm 1\%   } & {\tiny{} 55\% \textpm 2\%  } & {\tiny{} -1\% \textpm 1\%  } & {\tiny{} 2\% \textpm 0\%   } & {\tiny{} 6\% \textpm 1\%   } & {\tiny{} 0\% \textpm 1\%   } & {\tiny{} 0\% \textpm 2\%   } & {\tiny{} 4\% \textpm 0\%   } & {\tiny{} 1\% \textpm 1\%   } & {\tiny{} 0\% \textpm 2\%   }\tabularnewline
\hline 
{\tiny{}Gomoku} & {\tiny{}Havannah} & {\tiny{}Hex 11} & {\tiny{}Hex 13} & {\tiny{}Hex 19} & {\tiny{}Lines of A.} & {\tiny{}Othello 10} & {\tiny{}Othello 8} & {\tiny{}Santorini} & {\tiny{}Surakarta} & {\tiny{}Xiangqi}\tabularnewline
\hline 
{\tiny 0\% \textpm 1\%   } & {\tiny{} 0\% \textpm 2\%   } & {\tiny{} 0\% \textpm 1\%   } & {\tiny{} 0\% \textpm 1\%   } & {\tiny{} 0\% \textpm 1\%   } & {\tiny{} 5\% \textpm 1\%   } & {\tiny{} 24\% \textpm 1\%  } & {\tiny{} 28\% \textpm 1\%  } & {\tiny{} 14\% \textpm 2\%  } & {\tiny{} 4\% \textpm 1\%   } & {\tiny 11\% \textpm 1\% }\tabularnewline
\end{tabular}}{\tiny\par}
\end{table*}
}{\small\par}

We conducted an empirical evaluation of the Unbounded Best-First Minimax
algorithm with and without  completion. We will show that its use
has an impact on performance and that this impact is positive and
more precisely significantly positive, using standardized benchmarks
across several two-player games.

\subsection{Experimental Protocol}

We apply the experimental protocol for evaluating search algorithms
from the article~\cite{cohen2025study}. We use the same games, the
same evaluation functions, the same procedures, the same computers,
etc. All the details of this experiment are therefore available in
\cite{cohen2025study}.

Thus, we evaluated the two algorithms in the context of the 22 deterministic,
perfect-information, zero-sum games, namely Amazons, Ataxx, Breakthrough,
Brazilian Draughts, Canadian Draughts, Clobber, Connect6, International
Draughts, Chess, Go $9\times9$, (Outer-Open-)Gomoku, Havannah, Hex
$11\times11$, Hex $13\times13,$ Hex $19\times19$, Lines of Action,
Othello $10\times10$, Othello $8\times8$, Santorini, Surakarta and,
Xiangqi, and Arimaa. They are all recurrent games at the Computer
Olympiad, the global artificial intelligence board game competition.

For each game, six distinct sets of evaluation functions are used,
each containing approximately fifteen functions. Evaluation functions
are learned by reinforcement using the Athénan algorithm. The evaluation
process is repeated for each of these six sets. During each repetition,
the two algorithms compete against each other (once as the first player
and once as the second) for each pair of evaluation functions in the
current set. This results in roughly 450 matches per repetition, totaling
about 2,700 matches per game across all six repetitions. Each search
algorithm uses 10 seconds of search time per action (which is of the
same order of magnitude as times in competition).

All algorithms use transposition tables and the safe decision technique
(the action selected to play after performing the search by the algorithm
is the most frequently selected action at the end of the search)~\cite{cohen2025study}.

To assess the performance of a search algorithm, match outcomes are
aggregated using the following scoring: 1 for a win, -1 for a loss,
and 0 for a draw. The final performance on a game is the average score
across all matches played by the algorithm in that game. We report
95\% confidence intervals for these values, along with an overall
average performance across all games. In addition, we provide an unbiased
global confidence interval using stratified bootstrapping at the 95\%
level.

See also Technical Details of the Appendix.

\subsection{Results}

We thus evaluate Unbounded Minimax without completion against Unbounded
Minimax with completion. On average across all games, using completion
improves performance. The score with completion is $6.34\%\in\left[6.03\%,6.66\%\right]$
and the score without completion is $-6.34\%$. Detailed performances
are described in Table~\ref{tab:Detailed-performance-of}. These
results indicate that Completion has a strong positive impact on the
algorithm's behavior.

\section{Conclusion}

This article focuses on search algorithms for games, whose objective
is to determine the best possible strategy, and ideally a winning
strategy. Although determining a winning strategy for a game is in
general impractical due to the combinatorial explosion, it remains
an essential problem. Indeed, in the late game, the combinatorics
is reduced, which makes it possible to determine a winning strategy.
It is therefore essential that a search algorithm for games be able
to determine a good approximation of a winning strategy in the early
game, and be able to determine a winning strategy in the late game.
Otherwise, an algorithm that provides a worse approximation of a winning
strategy in the early game, but which would be able to determine a
winning strategy in the late game, would be able to regain the advantage
and win the game. The desired property is therefore that a search
algorithm is always able to determine a winning strategy (when it
exists) if it is given enough computation time. Such search algorithms
are called complete.

In this article, we have in particular studied the algorithms Unbounded
Best-First Minimax and Descent Minimax extended by the so-called completion
technique. These two algorithms are parts of Athénan, the state-of-the-art
of game playing and learning without knowledge in two-player games
with perfect information, being in particular the leading algorithm
in international game competitions, by a very wide margin. The completion
technique was proposed to improve the two algorithms, because the
latter two are not complete: they do not always determine a winning
strategy when it exists. The question then arose whether the completion
technique makes the two algorithms complete. If this was not the case,
then the two algorithms would have had to be further improved. This
article answers this problem. 

For that, we have generalized the two algorithms Unbounded Best-First
Minimax and Descent using the completion technique within Unbounded
Minimax-based algorithms. This class of algorithms includes all the
possibilities of behavior with regard to the decision to continue
the exploration in depth of the current line of play or to return
to the root to possibly choose a new strategy. Unbounded Best-First
Minimax always returns to the root and Descent returns to the root
only when the line of play reaches the end of the game. This class
of algorithms also includes all the possibilities of behavior for
the order of exploration of the different children of a state. Unbounded
Best-First Minimax and Descent always choose the best child.

Then, we have proved, in a unique proof, that any Unbounded Minimax-based
algorithm is complete, i.e. with sufficient time and memory, they
allow one to (weakly) solve perfect-information two-player games and
therefore to determine for these games a best strategy (a winning
strategy if any exists or otherwise a draw strategy if any exists).
Thus, Unbounded Best-First Minimax and Descent using the completion
technique are effectively complete.

However, a second question arises: does this modification actually
improve practical performance or would all of this merely be a negligible
detail? Is it important to implement this change? We have also addressed
this second question by experimentally showing that completion significantly
improves winning performance.

Note that some of the other Unbounded Minimax-based algorithms might
have practical value, at least in some specific contexts. We already
know that they can always be used to determine winning strategies.
In fact, some of these new algorithms are useful, as we will see in
a future article.

\bibliographystyle{named}
\bibliography{jeux}

@article{silver2017mastering,
  title={Mastering the game of Go without human knowledge},
  author={Silver, David and Schrittwieser, Julian and Simonyan, Karen and Antonoglou, Ioannis and Huang, Aja and Guez, Arthur and Hubert, Thomas and Baker, Lucas and Lai, Matthew and Bolton, Adrian and others},
  journal={Nature},
  volume={550},
  number={7676},
  pages={354},
  year={2017},
  publisher={Nature Publishing Group}
}

@book{mandziuk2010knowledge,
  title={Knowledge-free and learning-based methods in intelligent game playing},
  author={Mandziuk, Jacek},
  volume={276},
  year={2010},
  publisher={Springer}
}

@article{baier2018mcts,
  title={MCTS-Minimax Hybrids with State Evaluations},
  author={Baier, Hendrik and Winands, Mark HM},
  journal={Journal of Artificial Intelligence Research},
  volume={62},
  pages={193--231},
  year={2018}
}

@inproceedings{chaslot2008monte,
  title={Monte-Carlo Tree Search: A New Framework for Game AI.},
  author={Chaslot, Guillaume and Bakkes, Sander and Szita, Istvan and Spronck, Pieter},
  booktitle={AIIDE},
  year={2008}
}

@article{browne2012survey,
  title={A survey of monte carlo tree search methods},
  author={Browne, Cameron B and Powley, Edward and Whitehouse, Daniel and Lucas, Simon M and Cowling, Peter I and Rohlfshagen, Philipp and Tavener, Stephen and Perez, Diego and Samothrakis, Spyridon and Colton, Simon},
  journal={Transactions on Computational Intelligence and AI in games},
  volume={4},
  number={1},
  pages={1--43},
  year={2012},
  publisher={IEEE}
}

@article{baudet1978branching,
  title={On the branching factor of the alpha-beta pruning algorithm},
  author={Baudet, G{\'e}rard M},
  journal={Artificial Intelligence},
  volume={10},
  number={2},
  pages={173--199},
  year={1978},
  publisher={Elsevier}
}

@article{pearl1982solution,
  title={The solution for the branching factor of the alpha-beta pruning algorithm and its optimality},
  author={Pearl, Judea},
  journal={Communications of the ACM},
  volume={25},
  number={8},
  pages={559--564},
  year={1982},
  publisher={ACM}
}

@article{plaat1996best,
  title={Best-first fixed-depth minimax algorithms},
  author={Plaat, Aske and Schaeffer, Jonathan and Pijls, Wim and De Bruin, Arie},
  journal={Artificial Intelligence},
  volume={87},
  number={1-2},
  pages={255--293},
  year={1996},
  publisher={Elsevier}
}

@article{plaat2017minimax,
  title={A Minimax Algorithm Better Than Alpha-beta?: No and Yes},
  author={Plaat, Aske and Schaeffer, Jonathan and Pijls, Wim and De Bruin, Arie},
  journal={arXiv preprint arXiv:1702.03401},
  year={2017}
}

@article{korf1996best,
  title={Best-first minimax search},
  author={Korf, Richard E and Chickering, David Maxwell},
  journal={Artificial intelligence},
  volume={84},
  number={1-2},
  pages={299--337},
  year={1996},
  publisher={Elsevier}
}

@incollection{berliner1981b,
  title={The B* tree search algorithm: A best-first proof procedure},
  author={Berliner, Hans},
  booktitle={Readings in Artificial Intelligence},
  pages={79--87},
  year={1981},
  publisher={Elsevier}
}

@article{mcallester1988conspiracy,
  title={Conspiracy numbers for min-max search},
  author={McAllester, David Allen},
  journal={Artificial Intelligence},
  volume={35},
  number={3},
  pages={287--310},
  year={1988},
  publisher={Elsevier}
}

@article{schaeffer1990conspiracy,
  title={Conspiracy numbers},
  author={Schaeffer, Jonathan},
  journal={Artificial Intelligence},
  volume={43},
  number={1},
  pages={67--84},
  year={1990},
  publisher={Elsevier}
}

@inproceedings{kocsis2006bandit,
  title={Bandit based monte-carlo planning},
  author={Kocsis, Levente and Szepesv{\'a}ri, Csaba},
  booktitle={European conference on machine learning},
  pages={282--293},
  year={2006},
  organization={Springer}
}

@article{kocsis2006improved,
  title={Improved monte-carlo search},
  author={Kocsis, Levente and Szepesv{\'a}ri, Csaba and Willemson, Jan},
  journal={Univ. Tartu, Estonia, Tech. Rep},
  volume={1},
  year={2006}
}

@article{lanctot2014monte,
  title={Monte Carlo tree search with heuristic evaluations using implicit minimax backups},
  author={Lanctot, Marc and Winands, Mark HM and Pepels, Tom and Sturtevant, Nathan R},
  journal={arXiv preprint arXiv:1406.0486},
  year={2014}
}

@inproceedings{winands2008monte,
  title={Monte-Carlo tree search solver},
  author={Winands, Mark HM and Bj{\"o}rnsson, Yngvi and Saito, Jahn-Takeshi},
  booktitle={International Conference on Computers and Games},
  pages={25--36},
  year={2008},
  organization={Springer}
}

@inproceedings{baier2013monte,
  title={Monte-Carlo tree search and minimax hybrids},
  author={Baier, Hendrik and Winands, Mark HM},
  booktitle={Conference on Computational Intelligence in Games},
  pages={1--8},
  year={2013},
  organization={IEEE}
}

@article{baier2015mcts,
  title={MCTS-minimax hybrids},
  author={Baier, Hendrik and Winands, Mark HM},
  journal={Transactions on Computational Intelligence and AI in Games},
  volume={7},
  number={2},
  pages={167--179},
  year={2015},
  publisher={IEEE}
}

@phdthesis{lin2017monte,
  title={Monte Carlo Tree Search and Minimax Combination--Application of Solving Problems in the Game of Go},
  author={Lin, Jonathan Fun},
  year={2017}
}

@article{knuth1975analysis,
  title={An analysis of alpha-beta pruning},
  author={Knuth, Donald E and Moore, Ronald W},
  journal={Artificial Intelligence},
  volume={6},
  number={4},
  pages={293--326},
  year={1975},
  publisher={Elsevier}
}

@article{gelly2011monte,
  title={Monte-Carlo tree search and rapid action value estimation in computer Go},
  author={Gelly, Sylvain and Silver, David},
  journal={Artificial Intelligence},
  volume={175},
  number={11},
  pages={1856--1875},
  year={2011},
  publisher={Elsevier}
}

@book{allis1994searching,
  title={Searching for solutions in games and artificial intelligence},
  author={Allis, Louis Victor and others},
  year={1994},
  publisher={Ponsen \& Looijen Wageningen}
}

@book{russell2016artificial,
  title={Artificial intelligence: a modern approach},
  author={Russell, Stuart J and Norvig, Peter},
  year={2016},
  publisher={Malaysia; Pearson Education Limited,}
}

@inproceedings{Coulom06,
  author    = {R{\'{e}}mi Coulom},
  title     = {Efficient Selectivity and Backup Operators in Monte-Carlo Tree Search},
  booktitle = {Computers and Games, 5th International Conference, {CG} 2006, Turin,
               Italy, May 29-31, 2006. Revised Papers},
  pages     = {72--83},
  year      = {2006}
}

@article{silver2018general,
  title={A general reinforcement learning algorithm that masters chess, shogi, and Go through self-play},
  author={Silver, David and Hubert, Thomas and Schrittwieser, Julian and Antonoglou, Ioannis and Lai, Matthew and Guez, Arthur and Lanctot, Marc and Sifre, Laurent and Kumaran, Dharshan and Graepel, Thore and others},
  journal={Science},
  volume={362},
  number={6419},
  pages={1140--1144},
  year={2018},
  publisher={American Association for the Advancement of Science}
}

@article{fink1982enhancement,
  title={An Enhancement to the Iterative, Alpha-Beta, Minimax Search Procedure},
  author={Fink, William},
  journal={ICGA Journal},
  volume={5},
  number={1},
  pages={34--35},
  year={1982},
  publisher={IOS Press}
}

@article{
2020learning, title={Learning to Play Two-Player Perfect-Information Games without Knowledge}, author={Cohen-Solal, Quentin}, journal={arXiv preprint arXiv:2008.01188}, year={2020} }

@inproceedings{cohen2019apprendre,
  title={Apprendre {\`a} jouer aux jeux {\`a} deux joueurs {\`a} information parfaite sans connaissance},
  author={Cohen-Solal, Quentin},
  year={2019}
}

@article{cohen2020minimax,
  title={Minimax Strikes Back},
  author={Cohen-Solal, Quentin and Cazenave, Tristan},
  journal={arXiv preprint arXiv:2012.10700},
  year={2020}
}

@book{yannakakis2018artificial,
  title={Artificial intelligence and games},
  author={Yannakakis, Georgios N and Togelius, Julian},
  volume={2},
  year={2018},
  publisher={Springer}
}

@misc{bouzy2020artificial,
  title={Artificial Intelligence for Games},
  author={Bouzy, Bruno and Cazenave, Tristan and Corruble, Vincent and Teytaud, Olivier},
  year={2020},
  publisher={Springer}
}

@book{millington2018artificial,
  title={Artificial intelligence for games},
  author={Millington, Ian and Funge, John},
  year={2018},
  publisher={CRC Press}
}

@book{morgenstern1953theory,
  title={Theory of games and economic behavior},
  author={Morgenstern, Oskar and Von Neumann, John},
  year={1953},
  publisher={Princeton university press}
}

@article{pearl1980asymptotic,
  title={Asymptotic properties of minimax trees and game-searching procedures},
  author={Pearl, Judea},
  journal={Artificial Intelligence},
  volume={14},
  number={2},
  pages={113--138},
  year={1980},
  publisher={Elsevier}
}

@inproceedings{pearl1980scout,
  title={SCOUT: A Simple Game-Searching Algorithm with Proven Optimal Properties.},
  author={Pearl, Judea},
  booktitle={AAAI},
  pages={143--145},
  year={1980}
}

@inproceedings{qi2020optimization,
  title={Optimization of Connect6 based on Principal Variation Search and Transposition Tables Algorithms},
  author={Qi, Zhiyang and Huang, Xiaosong and Shen, Yu and Shi, Jiyun},
  booktitle={2020 Chinese Control And Decision Conference (CCDC)},
  pages={198--203},
  year={2020},
  organization={IEEE}
}

@article{donninger1993null,
  title={Null move and deep search},
  author={Donninger, Christian},
  journal={ICGA Journal},
  volume={16},
  number={3},
  pages={137--143},
  year={1993},
  publisher={IOS Press}
}

@article{schaeffer1983history,
  title={The history heuristic},
  author={Schaeffer, Jonathan},
  journal={ICGA Journal},
  volume={6},
  number={3},
  pages={16--19},
  year={1983},
  publisher={IOS Press}
}

@article{beal1990generalised,
  title={A generalised quiescence search algorithm},
  author={Beal, Don F},
  journal={Artificial Intelligence},
  volume={43},
  number={1},
  pages={85--98},
  year={1990},
  publisher={Elsevier}
}

@article{cohen2025study, 
title={Study and improvement of search algorithms in two-players perfect information games}, 
author={Cohen-Solal, Quentin}, 
journal={arXiv preprint arXiv:2505.09639}, 
year={2025} 
}

@article{cohen2023minimax,
  title={Minimax Strikes Back},
  author={Cohen-Solal, Quentin and Cazenave, Tristan},
  journal={AAMAS},
  year={2023}
}

@article{cohen2021descent,
	title={DESCENT wins five gold medals at the Computer Olympiad},
	author={Cohen-Solal, Quentin and Cazenave, Tristan},
	journal={ICGA Journal},
	volume={43},
	number={2},
	pages={132--134},
	year={2021},
	publisher={IOS Press}
}

@article{cohen2023athenan,
	title={Athenan Wins Sixteen Gold Medals at the Computer Olympiad},
	author={Cohen-Solal, Quentin and Cazenave, Tristan},
	journal={ICGA Journal},
	volume={45},
	number={3},
	year={2023},
}

\section*{Appendix}

This appendix provides technical details of the experiments (Section~\ref{subsec:Technical-Details}),
some remarks (Section~\ref{subsec:Remarks}), and the formal proofs
of the article (Section~\ref{proofs-1}).

\subsection{Technical Details}\label{subsec:Technical-Details}

For clarity and compactness in data tables, all game-specific performance
percentages are rounded to the nearest whole number.

We use the Adastra supercomputer: a node is based on 1 AMD Trento
EPYC 7A53 (64 cores 2.0 GHz processor), 256 Gio of DDR4-3200 MHz CPU
memory, 4 Slingshot 200 Gb/s NICs, 8 GPUs devices (4 AMD MI250X accelerator,
each with 2 GPUs) with a total of 512 Gio of HBM2. Note that 64 matches
are carried out in parallel on a node. Programs are written in Python
using tensorflow.

\subsection{Remarks}\label{subsec:Remarks}

Note that by removing the pruning procedure from these algorithms
(i.e. to set $r\left(s\right)=1$ when $\left|c\left(s\right)\right|=1$),
they can be used to strongly solve perfect-information two-player
games (and therefore to determine all the winning strategies for each
state of the game), which has a much higher computational cost.

\subsection{Completeness of Unbounded Minimax-Based Algorithms}\label{proofs-1}

We now show that any Unbounded Minimax-based algorithm is complete.

Note first that if a state is marked as resolved (i.e. $r(s)=1$),
its values do not change:
\begin{lem}
\label{lem:Si-avant-1} Let $\hat{s}$ be a state of a perfect two-player
game $G$.

If $r\left(\hat{s}\right)=1$ before an iteration of a Unbounded Minimax-based
algorithm on a certain state $\hat{s}'$ of $G$ then after the iteration,
$r\left(\hat{s}\right)$ and $c\left(\hat{s}\right)$ have not changed. 
\end{lem}

\begin{proof}
Let $\hat{s}$ be a state of $G$.

Remark first that if a state $s$ satisfies $r(s)=1$ then $s\in\Spartiel$.
Indeed, on the one hand, if $s$ is non-terminal, $r(s)$ has been
set to $1$ in the large else block. Thus, either there is an entry
in the second if block and then $s$ has been added to $\Spartiel$
or there is no entry in the second if block and then by the condition
of this if block, we have $s\in\Spartiel$. %If it has been set in the second if block, then either $b$ was true or $b$ was false. If $b$ was true, then before the entry in the second if block, there was an entry in the first if block.  Thus, $s$ has been added to $\Spartiel$. If $b$ was false, then there was no entry in the first if block. However, by the condition of this block, we nevertheless have $s \in \Spartiel$. 
On the other hand, if $s$ is terminal, there are two places where
$r(s)$ can be set to $1$, and in each of these places, $s$ is added
to $\Spartiel$.

Secondly remark the following fact. Let $s'$ be a non-terminal child
state of any state $s$ of the sequence of states analysed by an iteration
of a Unbounded Minimax-based algorithm. If during the iteration, the
values $c(s')$ and $r(s')$ are updated, then $s'\notin\Spartiel$
(see the else block in the foreach block in Algorithm 4). Since the
state $\hat{s}$ satisfies $r\left(\hat{s}\right)=1$ by hypothesis,
we have $\hat{s}\in\Spartiel$ and thus, values of $\hat{s}$ cannot
be modified by this part of the algorithm.

Moreover, as a recursive call of a Unbounded Minimax-based algorithm
is only performed on a state $s$ such that $r(s)=0$ (since $s$
belongs to $A$), a modification by the algorithm of the values $r(\hat{s})$
and $c(\hat{s})$ necessarily implies that either this state $\hat{s}$
is the initial state $\hat{s}'$ (that of the non-recursive call of
the algorithm, i.e. the first call in the recursion) or that $\hat{s}$
is terminal (since we have $r(\hat{s})=1$). Thus, there are two cases
to analyze.

On the one hand, if $\hat{s}$ is terminal, %as it satisfies $r(\hat{s})=1$, 
when the values $r(\hat{s})$ and $c(\hat{s})$ are updated by the
algorithm, they are not modified (since there are only two places
in the algorithm where the values of a terminal state are assigned
and these assignments are equal and constant for each state).

On the other hand, suppose $\hat{s}$ is the initial state and is
non-terminal. %Since $r(\hat{s})=1$, we have, $\hat{s}\in \Spartiel$.
Since $r(\hat{s})=1$, there is no entry in the third if block. Thus,
the only place where the values can be modified is just before the
third if block. There is two cases: either $r(s')=1$ for each $s'\in\actions(\hat{s})$
or $c(s)=1$ and there exists $s'\in\actions(\hat{s})$ such that
$c(\hat{s})=c(s')$ and $r(s')=1$. By induction, the values $r(s')$
and $c(s')$ of child states safisfying $r(s')=1$ have not been modified
and thus the values $r(\hat{s})$ and $c(\hat{s})$ have not been
modified.

%Therefore, applying the algorithm will not cause any modification of the values $r(\hat{s})$ and $c(\hat{s})$, since there is neither entry in the first if block (as $\hat{s}\in \Spartiel$) nor in the second if block (as $r(\hat{s})=1$) and since these two blocks (with the else block in the foreach block previously analyzed) are the only places where the values a non-terminal state can be modified.
%
%
%%%%%%%%%% NOTE %%%%%%%%%%%
%
%On pourrait remplacer la condition $s'\notin \Spartiel$ dans le else block par $r(s)=0$, cela serait moins performant comme algorithme mais d'après toutes mes preuves, cela marcherait quand même, car même si revient en arriere (raz valeurs) de non résolu, ajoute quand même un résolu ou aggrandit Sp à chaque itération, et donc termine forcement et les valeurs résolus sont correctes. Autre argument pour me convaincre, l'écrasement n'a lieu que dans un état non dans Sp, mais à un moment ils sont tous dans Sp donc plus écrasé et les prochaines propagations vont faire remonter les bonnes valeurs, comme les resolus intouchable par ce lemme.
\end{proof}
Note now that a state $s$ is marked as resolved (i.e. $r(s)=1$)
if and only if all of its children are marked as resolved (any child
state $s'$ satisfies $r(s')=1$) or $s$ is marked as winning for
the first player ($c(s)=1$) or it is marked as losing for the first
player ($c(s)=-1$):
\begin{lem}
\label{lem:completion-resolution-1}Let $\left(\Spartiel,\actions\right)$
be a game tree built by a Unbounded Minimax-based algorithm from a
certain state. Let $s\in\Spartiel$. We have the following properties: 
\begin{itemize}
\item If $r\left(s\right)=1$, then either $\left|c\left(s\right)\right|=1$
or for all $s'\in\actions\left(s\right)$, $r\left(s'\right)=1$ ; 
\item If $\left|c\left(s\right)\right|=1$, then $r\left(s\right)=1$ ; 
\item If for all $s'\in\actions\left(s\right)$, $r\left(s'\right)=1$,
then $r\left(s\right)=1$. 
\end{itemize}
\end{lem}

\begin{proof}
By definition of the algorithm (in particular by the definition and
by the use of the method backup\_resolution($s$) and because as soon
as $r\left(s\right)=1$, $r\left(s\right)$ and $c\left(s\right)$
do not change anymore (Lemma~\ref{lem:Si-avant-1})). 
\end{proof}
We now show that if a state $s$ is marked as resolved ($r(s)=1$)
then it is resolved, i.e. its completion value is equal to its minimax
value ($c\left(s\right)=M\left(s\right)$):
\begin{prop}
\label{prop:exact-1}Let $\left(\Spartiel,\actions\right)$ be a game
tree built by a Unbounded Minimax-based algorithm from a certain state
and let $s\in\Spartiel$.

If $r\left(s\right)=1$, then $c\left(s\right)=M\left(s\right)$. 
\end{prop}

\begin{proof}
Let $\left(\Spartiel,\actions\right)$ be a game tree built by a Unbounded
Minimax-based algorithm from a certain state (i.e. the algorithm has
been applied $k$ times on that state). We show this property by induction.
Let $s\in\Spartiel$ such that $r\left(s\right)=1$. We first show
that this property is true for terminal states.

Thus, suppose in addition that $\terminal\left(s\right)$ is true.
Therefore, $c\left(s\right)=\fbin\left(s\right)$. Consequently, $c\left(s\right)=\fbin\left(s\right)=M\left(s\right)$.

We now show this property for non-terminal states: therefore suppose
instead that $\terminal\left(s\right)$ is false. 
\begin{itemize}
\item Suppose on the one hand that $\joueur\left(s\right)=\joueurUn$. 
\end{itemize}
If $r\left(s\right)=1$, then either $\left|c\left(s\right)\right|=1$
or for all $s'\in\actions\left(s\right)$, $r\left(s'\right)=1$,
by Lemma~\ref{lem:completion-resolution-1}. If $c\left(s\right)=-1$,
then, at the iteration that calculated $c\left(s\right)=-1$, we had:
for all $s'\in\actions\left(s\right)$, $c\left(s'\right)=-1$ (as
at this moment $c\left(s\right)=\max_{s'\in\actions\left(s\right)}c\left(s'\right)$).
Therefore, since this iteration, we have for all $s'\in\actions\left(s\right)$,
$r\left(s'\right)=1$ (by Lemma~\ref{lem:completion-resolution-1}
and Lemma~\ref{lem:Si-avant-1}). Thus, we have in fact two cases:
either $c\left(s\right)=1$ or for all $s'\in\actions\left(s\right)$,
$r\left(s'\right)=1$.

If for all $s'\in\actions\left(s\right)$, $r\left(s'\right)=1$,
then, by induction, we have for all $s'\in\actions\left(s\right)$,
$c\left(s'\right)=M\left(s'\right)$. But $c\left(s\right)=\max_{s'\in\actions\left(s\right)}c\left(s'\right)$.
Therefore, $c\left(s\right)=\max_{s'\in\actions\left(s\right)}M\left(s'\right)$,
hence $c\left(s\right)=M\left(s\right)$.

If $c\left(s\right)=1$, then there exists $\tilde{s}\in\actions\left(s\right)$
such that $c\left(\tilde{s}\right)=c\left(s\right)$ at the iteration
calculating $c\left(s\right)=1$ (as $c\left(s\right)=\max_{s'\in\actions\left(s\right)}c\left(s'\right)$
at this iteration). Since we had $c\left(\tilde{s}\right)=1$ at this
iteration, we also had $r\left(\tilde{s}\right)=1$, by Lemma~\ref{lem:completion-resolution-1}.
Therefore, we always have $c\left(\tilde{s}\right)=c\left(s\right)$
and $r\left(\tilde{s}\right)=1$. Moreover, by induction, $c\left(\tilde{s}\right)=M\left(\tilde{s}\right)$,
hence $c\left(s\right)=M\left(\tilde{s}\right)$. However, since $M\left(\tilde{s}\right)=1$,
we have $M\left(\tilde{s}\right)\geq M\left(s'\right)$ for all $s'\in\actions\left(s\right)$.
Thus, $c\left(s\right)=\max_{s'\in\actions\left(s\right)}M\left(s'\right)$,
hence $c\left(s\right)=M\left(s\right)$. 
\begin{itemize}
\item Now suppose on the other hand, instead that $\joueur\left(s\right)=\joueurDeux$.
(this case is analogous to the previous case, we detail it anyway
in the rest of the proof). 
\end{itemize}
If $r\left(s\right)=1$ then either $\left|c\left(s\right)\right|=1$
or for all $s'\in\actions\left(s\right)$, $r\left(s'\right)=1$,
by Lemma~\ref{lem:completion-resolution-1}. If $c\left(s\right)=1$,
then at the iteration that calculated $c\left(s\right)=1$, we had:
for all $s'\in\actions\left(s\right)$, $c\left(s'\right)=1$ (as
at this moment $c\left(s\right)=\min_{s'\in\actions\left(s\right)}c\left(s'\right)$).
Therefore, since this iteration, for all $s'\in\actions\left(s\right)$,
$r\left(s'\right)=1$ (by Lemma~\ref{lem:completion-resolution-1}
and Lemma~\ref{lem:Si-avant-1}). Thus, we have two cases: either
$c\left(s\right)=-1$ or for all $s'\in\actions\left(s\right)$, $r\left(s'\right)=1$.

If for all $s'\in\actions\left(s\right)$, $r\left(s'\right)=1$,
then, by induction, we have for all $s'\in\actions\left(s\right)$,
$c\left(s'\right)=M\left(s'\right)$. But $c\left(s\right)=\min_{s'\in\actions\left(s\right)}c\left(s'\right)$.
Therefore $c\left(s\right)=\min_{s'\in\actions\left(s\right)}M\left(s'\right)$,
hence $c\left(s\right)=M\left(s\right)$.

If $c\left(s\right)=-1$, then there exists $\tilde{s}\in\actions\left(s\right)$
such that $c\left(\tilde{s}\right)=c\left(s\right)$ at the iteration
calculating $c\left(s\right)=-1$ (as $c\left(s\right)=\min_{s'\in\actions\left(s\right)}c\left(s'\right)$
at this iteration). Since we had $c\left(\tilde{s}\right)=-1$ at
this iteration, we also had $r\left(\tilde{s}\right)=1$, by Lemma~\ref{lem:completion-resolution-1}.
Therefore, we always have $c\left(\tilde{s}\right)=c\left(s\right)$
and $r\left(\tilde{s}\right)=1$. Moreover, by induction, $c\left(\tilde{s}\right)=M\left(\tilde{s}\right)$,
hence $c\left(s\right)=M\left(\tilde{s}\right)$. However, since $M\left(\tilde{s}\right)=-1$,
we have $M\left(\tilde{s}\right)\leq M\left(s'\right)$ for all $s'\in\actions\left(s\right)$.
Thus, $c\left(s\right)=\min_{s'\in\actions\left(s\right)}M\left(s'\right)$,
hence $c\left(s\right)=M\left(s\right)$. 
\end{proof}
In the lemma below, we show under which sufficient conditions an iteration
of a Unbounded Minimax-based algorithm terminates.
\begin{lem}
\label{lem:arret-1}Let $\algo$ be a Unbounded Minimax-based algorithm
and $s'$ be a non-terminal state of a perfect two-player game such
that $r(s')=0$.

An iteration of the algorithm $\algo$ on the state $s'$ ends in
a state $s$ such that either $s\notin\Spartiel$ before the iteration
and $s\in\Spartiel$ after the iteration or $r(s)=0$ before the iteration
and $r(s)=1$ after the iteration. 
\end{lem}

\begin{proof}
\begin{comment}
Supposons que $s\neq s'$. Ainsi, $s$ appartient à l'ensemble des
actions $A$ d'un certain état $\hat{s}$ ($\hat{s}$ est l'état père
de $s$, état précédent dans la récursion). Par conséquent, nous avons
$r(s)=0$. En outre, nécessairement avant l'entrée dans le second
if block qui permet le calcul du $A$ de $\hat{s}$, il y a eut une
itération $I'$ où $\hat{s}$ a été ajouté à $\Spartiel$, dans le
premier if block. Ainsi, lors de l'entrée dans ce premier if block,
si $s$ est terminal, $r(s)$ a été posé par $1$. Par contraposé,
si $r(s)=0$ à la sortie de ce block, alors $s$ n'est pas terminal.
Puisque $r(s)=0$ au début de l'itération $I$ (et que si $r(s)=1$,
$r(s)$ reste constant par le Lemme \ref{lem:Si-avant-1}), $s$ n'est
donc pas terminal. Cela conclut la preuve du lemme.
\end{comment}

%%%%%%%%%%%%%%%%%

Let $s$ be the state in which ends an iteration, denoted $I$, of
a Unbounded Minimax-based algorithm applied to a non-terminal state
$s'$ satisfying $r(s')=0$.

Note first that when the algorithm terminates, necessarily either

$s$ is terminal, $r(s)=1$ after the iteration $I$, or $b$ is true,
i.e. $s\notin\Spartiel$ before the iteration and $s\in\Spartiel$
after the iteration.

We first show that before the iteration $I$, we still have $r(s)=0$.
Two cases must be analysed. The first is when the algorithm ends in
the initial state $s'$. By assumption, we have $r(s)=0$. The second
case is when the algorithm does not end in the initial state, which
implies that the state $s$ belongs to the set of actions $A=\liste{s''\in\actions\left(\hat{s}\right)}{r\left(s''\right)=0}$
for some state $\hat{s}$. We thus also have $r(s)=0$. Note that
$A$ is never empty, since when there is no more child $s''$ of $\hat{s}$
satisfying $r(s'')=0$, $r(\hat{s})$ has been set to $1$ and thus
there is no more entrance in the third if block where $A$ is computed.

To finish showing the lemma, we just have to show that the last state
$s$ of the iteration $I$ is never terminal. If the iteration stops
in the initial state $s'$, by hypothesis, we have $s=s'$ is never
terminal.

Suppose $s\neq s'$. Thus, $s$ belongs to the set of actions $A$
of a certain state $\hat{s}$ ($\hat{s}$ is the parent state of $s$,
i.e. the previous state in the recursion). Therefore, we have $r(s)=0$.
In addition, necessarily before the entry into the third if block
which allows one the calculation of the set $A$ of $\hat{s}$, there
was an iteration $I'$ (potentially $I=I'$) where $\hat{s}$ was
added to $\Spartiel$. More precisely, this addition, during the iteration
$I'$, necessarily takes place in the second if block (since $\hat{s}$
is not terminal). %  parce qu'il est passé dans le large else block où A sera ensuite calculé et lorsqu'un état est non terminal, le seul endroit où on peut j'ajoute dans Spartial c'est dans ce block.
 Thus, when entering this second if block, if $s$ is terminal, $r(s)$
was set by $1$. By contraposition, if $r(s)=0$ at the output of
this block, then $s$ is not terminal. Since $r(s)=0$ at the beginning
of the iteration $I$ (and if $r(s)=1$, $r(s)$ remains constant
by Lemma \ref{lem:Si-avant-1}), $s$ is therefore not terminal. This
concludes the proof of the lemma.
\end{proof}
We now show that, all Unbounded Minimax-based algorithm mark any game
state as resolved after a finite number of iterations:
\begin{prop}
\label{prop:arret-1} Let $\algo$ be a Unbounded Minimax-based algorithm,
$\etats$ be the set of states of a perfect two-player game, and $s\in\etats$.

There exists $N\in\mathbb{N}$ such that after applying $N$ times
the algorithm $\algo$ on the state $s$, we have $r\left(s\right)=1$. 
\end{prop}

\begin{proof}
We show that with at most $N=2\left|\etats\right|$ iterations of
the algorithm $\algo$ applied on a same state $s\in\etats$, we have
$r\left(s\right)=1$. Note first that if $s$ is terminal or satisfies
$r\left(s\right)=1$, then after having applied the algorithm, we
have $r\left(s\right)=1$.

Now suppose that $s$ is not terminal and satisfies $r\left(s\right)=0$.
By Lemma \ref{lem:arret-1}, we have that each iteration adds in $\Spartiel$
at least one state of $\etats$ not being in $\Spartiel$ or marks
as ``resolved'' an additional state, that is, a state $s'\in\etats$
satisfying $r\left(s'\right)=0$, satisfies $r\left(s'\right)=1$
after the iteration. This is sufficient to show the property, because
either after one of the iterations, we have $r\left(s\right)=1$,
or the iterative application of the algorithm ends up adding in $\Spartiel$
all descendants of $s$ and/or by marking as ``resolved'' all states
of $\Spartiel$. Indeed, if all the descendants of $s$ are added
then necessarily $r\left(s\right)=1$ (since by induction all descendants
verify $r\left(s\right)=1$ ; by definition and use of backup\_resolution($s$)).
Since $S$ is finite, with at most $2\left|\etats\right|$ iterations,
$r\left(s\right)=1$.

\end{proof}
From the previous Lemmas and Propositions, we conclude to the completeness
of Unbounded Minimax-based algorithms:
\begin{thm}
Any Unbounded Minimax-based algorithm is complete. 
\end{thm}

\begin{proof}
By Proposition~\ref{prop:arret-1}, then by Proposition~\ref{prop:exact-1}. 
\end{proof}

\end{document}